\documentclass[final]{cvpr}
\usepackage[document]{ragged2e}
\usepackage{times}
\usepackage{epsfig}
\usepackage{graphicx}
\usepackage{amsmath}
\usepackage{amssymb}
\usepackage{color}
\usepackage{algorithm}
\usepackage{algorithmicx}
\usepackage{subfigure}
\usepackage{algpseudocode}
\usepackage{threeparttable}  
\usepackage{multirow}  
\usepackage{tabu}  
\usepackage{float}
\usepackage{caption}
\usepackage{amsthm}
\newcommand{\new}{\text{new}}
\newcommand{\old}{\text{old}}

\newcommand{\rank}{\text{rank}}

\newtheorem{lemma}{Lemma}

\usepackage[pagebackref=true,breaklinks=true,letterpaper=true,colorlinks,bookmarks=false]{hyperref}

\def\cvprPaperID{6235} % *** Enter the CVPR Paper ID here
\def\confYear{CVPR 2021}

\begin{document}

%%%%%%%%% TITLE
\title{Layerwise Optimization by Gradient Decomposition for Continual Learning}

\author{Shixiang Tang$^1$\footnotemark[2] \qquad Dapeng Chen$^3$ \qquad Jinguo Zhu$^2$ \qquad Shijie Yu$^4$ \qquad Wanli Ouyang$^1$ \\
$^1$The University of Sydney, SenseTime Computer Vision Group, Australia \qquad $^2$Xi'an Jiaotong University \\
$^3$Sensetime Group Limited, Hong Kong \qquad $^4$Shenzhen Institutes of Advanced Technology, CAS \\  
{\tt\small tangshixiang@sensetime.com \quad  dapengchenxjtu@yahoo.com \quad  wanli.ouyang@sydney.edu.au }
}

% \and

\maketitle
\thispagestyle{empty}
\footnotetext[2]{This work was done when Shixiang Tang was an intern at SenseTime.}

%%%%%%%%% ABSTRACT
\begin{abstract}
Deep neural networks achieve state-of-the-art and sometimes super-human performance across various domains. However, when learning tasks sequentially, the networks easily forget the knowledge of previous tasks, known as ``catastrophic forgetting''. To achieve the consistencies between the old tasks and the new task, one effective solution is to modify the gradient for update. Previous methods enforce independent gradient constraints for different tasks, while  we consider these gradients contain complex information, and propose to leverage inter-task information by gradient decomposition. In particular, the gradient of an old task is decomposed into a part shared by all old tasks and a part specific to that task. The gradient for update should be close to the gradient of the new task, consistent with the gradients shared by all old tasks, and orthogonal to the space spanned by the gradients specific to the old tasks. In this way, our approach encourages common knowledge consolidation without impairing the task-specific knowledge. Furthermore, the optimization is performed for the gradients of each layer separately rather than  the concatenation of all gradients as in previous works. This effectively avoids the influence of the magnitude variation of the gradients in different layers. Extensive experiments validate the effectiveness of both gradient-decomposed optimization and layer-wise updates. Our proposed method achieves state-of-the-art results on various benchmarks of continual learning. 

\end{abstract}

%%%%%%%%% BODY TEXT
\section{Introduction}
Recent years have witnessed the great progress in deep learning (DL) through training on large datasets. A typical supervised DL task requires independent and identically distributed (\textit{i.i.d}) samples from a stationary distribution~\cite{he2016deep,ren2015faster,simonyan2014very}. However, many DL models deployed in the real-world are exposed to non-stationary situations where data is acquired sequentially and its distribution varies over time. In this scenario, DNNs trained by Stochastic Gradient Descent (SGD) will easily forget the knowledge from previous tasks while adapting to the information from the incoming tasks. This phenomenon, known as \textit{catastrophic forgetting}, invokes more effective algorithms for continual learning (CL), the goal of which is to learn consecutive tasks without severe performance degradation on previous tasks~\cite{castro2018end,lopez2017gradient,parisi2019continual,rebuffi2017icarl,shmelkov2017incremental,shin2017continual,zhao2021continual,zhao2021continual,tang2021gradient}.

One of the popular attempts for continual learning relies on a set of episodic memories, where each episodic memory stores representative data from an old task~\cite{castro2018end,rebuffi2017icarl,lopez2017gradient}. In particular, the network parameters are jointly optimized by the recorded samples that are regularly replayed and the samples drawn from the new task. An effective solution is to modify the gradients for updates. The methods like GEM, A-GEM and S-GEM obtain gradient update by requiring the loss of each old task does not increase~\cite{chaudhry2018efficient,lopez2017gradient}. In practice, the constraints are imposed by forcing the inner product between the gradient for update and the gradient of every old task non-negative. 

In this work, we consider that the gradients of multiple tasks have mixed information and should be disentangled when used for learning. There are two parts mixed in the gradients, shared gradient and task-specific gradients. Optimization along the shared gradient will be beneficial for memorizing all old tasks, but optimization along task-specific gradients will fall into a dilemma that optimizing one episodic memory will inevitably damage another. Therefore, different constraints are imposed on two gradients. We encourage the consistency between gradient for update and the shared gradient but expect that the gradient for update will be orthogonal to all task-specific gradients. The first constraint can be mathematically presented as the inner product of shared gradient and gradient for update non-negative. The second constraint requires the gradient for update should be orthogonal to the space spanned by all task-specific gradients. Our further analysis points out that the second constraints can be relaxed by PCA, which captures the most important gradient constraints.  

Furthermore, we observe the large variation of magnitudes for the gradients in different layers. However, the previous gradient modification methods ignore the intrinsic magnitude variations of different layers. They concatenate the gradients from different layers into a vector and construct optimization problems under the constraints made by these concatenated vectors. We argue that during optimization, gradients from some layers have larger magnitudes and they will dominate the solution of the optimization problem. However, no evidence shows that these layers are more important for loss minimization than the others. To address this problem, we propose a layer-wise gradient update strategy, where unique gradient constraints are imposed by each layer for optimization and the solution is only specific to the parameters in that layer. Our further analysis manifests that the layer-wise optimization strategy increases the efficiency of reducing old task losses
The contribution of this paper is two-fold: (1) Gradient decomposition is leveraged to specify the shared and task-specific information in the episodic memory. Different constraints are imposed based on the shared gradient and tasks-specific gradient respectively. (2) Layer-wise gradient update strategy is proposed to deal with large magnitude variations between gradients from different layers and thus it can reduce the losses of episodic memory more efficiently. Extensive ablation studies validate the effectiveness of the two improvements.

\section{Related Work}

\subsection{Continual Learning}
Continual learning has been a long-standing research problem in the field of machine learning~\cite{mccloskey1989catastrophic,silver2013lifelong,pentina2014pac,sodhani2018training}. Generally speaking, there are three different types of scenarios for continual learning: (1) class-incremental scenario, where the number of class labels keeps growing but no explicit task boundary is defined under test; (2) task-incremental scenario, where the boundaries among tasks are assumed known and the information about the task under test is given.
(3) data-incremental scenario, where the set of class labels are identical for all tasks. Existing work of above three scenarios mainly falls into two categories: Regularization-based and Memory-based approaches. 

\noindent \textbf{Regularization-based Approach:} Regularization approaches do not require data from old tasks. Rather, they attempted to tackle catastrophic forgetting by discouraging changes on important parameters or penalizing the change of activations on old tasks. The former approaches, such as EWC~\cite{kirkpatrick2017overcoming}, SI~\cite{zenke2017continual}, MAS~\cite{aljundi2018memory} and ALASSO~\cite{park2019continual} relied on different estimations of parameters importance, which was usually conducted by exhaustive search~\cite{rusu2016progressive} or variational Bayesian methods~\cite{nguyen2017variational,zeno2018task}. The latter approaches, such as LWF~\cite{li2017learning}, LFL~\cite{jung2016less}, LWM~\cite{dhar2019learning} and BLD~\cite{fini2020online}, leveraged knowledge distillation~\cite{hinton2015distilling,CRCD} among consecutive tasks.

\noindent \textbf{Memory-based Approach:} 
Memory-based approaches leverage episodic memory that stores representative samples from each old task to overcome catastrophic forgetting. One popular framework is multitask learning, where the model shares the same backbone but is equipped with different task-specific classifiers. During training, all episodic memory are mixed to form an old task and the loss of the model is defined by both old and new tasks. Example methods of this framework were iCarl~\cite{rebuffi2017icarl}, End2End~\cite{castro2018end} and DR~\cite{hou2018lifelong}. Under this framework, various problems were considered, such as data imbalance~\cite{hou2019learning,park2019continual}, using unlabeled data~\cite{lee2019overcoming} and the bias of fully-connected layers~\cite{wu2019large}. Another popular framework for memory-based approaches is to use episodic memory tasks to construct the optimization problems. GEM~\cite{lopez2017gradient} was the first work under this framework. In the method, each episodic memory constructed one constraint independently. During training the new task, GEM ensured the loss of every episodic memory non-increase. Further improvements on GEM such as A-GEM~\cite{chaudhry2018efficient} and S-GEM~\cite{chaudhry2018efficient} relaxed the constraints by either considering the loss averaged on all old tasks or constraining the loss of a random episodic memory every training iteration. Although these methods improved GEM in different aspects, they treated every episodic memory independently and no further analysis or decomposition of these gradients was considered in these works. 

\subsection{Layerwise Gradient Update}
Stochastic Gradient Descent is the most widely used optimization techniques for training DNNs~\cite{bottou2010large,loshchilov2016sgdr,bordes2009sgd}. However, it applied the same hyper-parameters to update all parameters in different layers, which may not be optimal for loss minimization. Therefore, layerwise adaptive optimization algorithms were proposed~\cite{duchi2011adaptive,kingma2014adam}. RMSProp~\cite{ruder2016overview} altered the learning rate of each layer by dividing the square root of its exponential moving average. LARS~\cite{you2017large} let the layerwise learning rate be proportional to the ratio of the norm of the weights to the norm of the gradients.
Both layerwise adaptive optimizers solved the variation of update frequencies for different layers and thus outperformed SGD in various large-scale benchmarks. Layerwise gradient update strategy is also applied in meta-learning for rapid loss convergence on transfer learning~\cite{ge2020mutual,ge2020selfpaced,su2020adapting} and few-shot learning~\cite{MutualCRF,Li_2019_ICCV}. \cite{flennerhag2019meta} proposed WarpGrad which layerwisely meta-learned to warp task loss surfaces across the joint task-parameter distribution to facilitate gradient descent. MT-Net~\cite{lee2018gradient} enabled the meta-learner to learn on each layer's activation space, a subspace that the task-specific learner performed gradient descent on. Our method differs from the above methods in two aspects. Our layerwise gradient update strategy aims at preserving old knowledge. Instead, RMSProp, LARS aimed to learn new tasks more efficiently, and WarpGrad, MT-Net aimed to transfer more related information in source data to target samples, which have different targets with our method. Secondly, RMSProp and LARS tried to handle the different update frequency between parameters in different layers but our method aims at handling large magnitude variation of parameters in different layers.

%------------------------------------------------------------------------
\section{Methodology}
For continual learning, we first define a task sequence $\{ \mathcal{T}_{1}, \mathcal{T}_{2},\cdots, \mathcal{T}_{N} \}$ of $N$ tasks. For the $t$-th task $\mathcal{T}_t$, there is a training dataset $\mathcal{D}_{t}$ and a memory coreset $\mathcal{M}_{t}^{\text{cor}}$. In particular, $\mathcal{D}_{t}=\{(x_t^{i}, y_t^{i})\}_{i=1}^{n_t}$, where each instance $(x_t^i, y_t^i)$ is composed of an image $x_t^i\in \mathcal{X}_t$ and a label $y_t^i \in \mathcal{Y}_t$. The coreset is represented by $\mathcal{M}_{t}^{\text{cor}}=\mathcal{M}_{t-1}^{\text{cor}}\cup\mathcal{M}_{t}$, where $\mathcal{M}_{t}$ is the episodic memory of task $\mathcal{T}_t$ that stores representative data from  $\mathcal{D}_{t}$.

The goal at the $t$-th step is to train a function $f$ to perform the current task $\mathcal{T}_t$ as well as the previous tasks $\mathcal{T}_{1:(t-1)}=\{\mathcal{T}_{1}, \mathcal{T}_{2},\cdots, \mathcal{T}_{t-1}\}$. During training, all data in $\mathcal{D}_{t}$ and $\mathcal{M}_{t-1}^{\text{cor}}$ are used. Considering $f$ is parameterized by $\theta_t$, we define the losses for new task $\mathcal{T}_t$ and any previous task $\mathcal{T}_i(i<t)$ using the training data $\mathcal{D}_t$ and episodic memory $\mathcal{M}_i$ respectively, i.e.,
\begin{equation}
\begin{split}
      \mathcal{L}^{\new}( \theta_{t}, \mathcal{D}_t) &= \frac{1}{|\mathcal{D}_t|}\sum_{(x, y) \in \mathcal{D}_t}CE(f(x;\theta_{t}), y) \\
      \mathcal{L}^{\old}(\theta_t, \mathcal{M}_{i}) &= \frac{1}{|\mathcal{M}_{i}|}\sum_{(x, y) \in \mathcal{M}_{i} }CE(f(x;\theta_{t}), y). 
\end{split}
\label{eq:loss_def}
\end{equation}
where $CE$ is the standard cross-entropy loss. The gradient of new task is defined as:
\begin{equation}
     g = \frac{ \partial  \mathcal{L}^{\new}( \theta_{t}, \mathcal{D}_t) }{ \partial \theta_{t}}.
\end{equation}

\begin{figure*}
\centering
\includegraphics[width=0.95\linewidth]{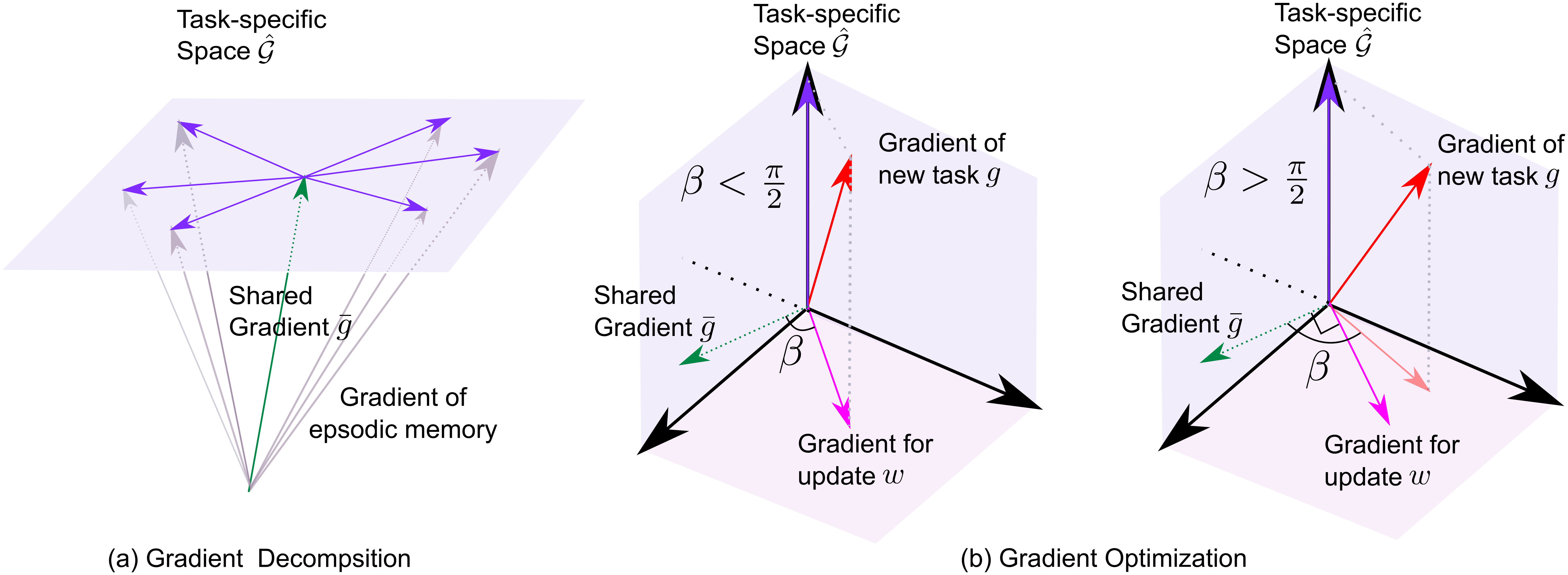}
\caption{\small{(a) Illustration of gradient decomposition. Light grey arrow: gradients of episodic memory. Green arrow: shared gradient $\bar{g}$. Purple arrow: task-specific gradients. Every gradient of episodic memory equals to the sum of the shared gradient $\bar{g}$ and its specific gradient $\hat{g}$. Task-specific Space $\hat{\mathcal{G}}$ is spanned by all task-specific gradients. (b) Illustration of gradient optimization. $\beta$ is the angle between the shared gradient $\bar{g}$ and the projection of $g$ onto the null space of $\hat{\mathcal{G}}$. $\beta < \frac{\pi}{2}$ and $\beta > \frac{\pi}{2}$ means condition $\bar{g}^{\top}Pg \ge 0$ and $\bar{g}^{\top}Pg < 0$ respectively. If $\beta < \frac{\pi}{2}$, the gradient for update $w$ is the projection of $g$ onto the null space of $\hat{\mathcal{G}}$. If $\beta > \frac{\pi}{2}$, the gradient for update $w$ should be orthogonal to both task-specific space $\hat{\mathcal{G}}$ and the shared gradient $\bar{g}$.}} 
\label{fig:algo}
\vspace{-0.5cm}
\end{figure*}

\subsection{Gradient Decomposition}
Biologically, continual learning is inspired by human learning. We learn the new tasks by applying the knowledge previously learned from the related tasks. In this way, there exists common knowledge shared among the old tasks. Given the losses of old task $\{\mathcal{L}^{\old}(\theta_t, \mathcal{M}_i)\}_{i=1}^{t-1}$, we assume each loss has two components. One is the shared loss, which is the same for any previous task. Minimizing the shared loss would improve the overall performance across all previous tasks. The other is task-specific component, reflecting the task-specific knowledge of every old task. 
For $\mathcal{L}^{old}(\theta_{t}, \mathcal{M}_{i})$ about the $i$-th episodic memory, it can be written as:
\begin{equation}
    \mathcal{L}^{old}( \theta_{t}, \mathcal{M}_i) = \mathcal{L}_{\text{shared}}^{old}(\theta_{t}, \mathcal{M}_{t-1}^{cor}) + \mathcal{R}^{old}_{\text{specific}}(\theta_{t}, \mathcal{M}_i)
\end{equation}
where $\mathcal{L}_{\text{shared}}^{old}(\theta^{t}, \mathcal{M}_{t-1}^{cor}) $ is the shared loss term, which is the same for different old tasks, and $\mathcal{R}^{old}_{\text{specific}}(\theta_{t}, \mathcal{M}_i)$ is the residual for the task $\mathcal{T}_{i}$. 

\noindent \textbf{Shared Gradient.} The shared gradient is driven by $\mathcal{L}_{\text{shared}}^{old}(\theta^{t}, \mathcal{M}_{t-1}^{cor})$. We define the shared loss by averaging all losses in the memory coreset $\mathcal{L}_{\text{shared}}^{old}(\theta^{t}, \mathcal{M}_{t-1}^{cor})$\\
$=\frac{1}{t-1} \sum_{i=1}^{t-1} \mathcal{L}^{old}(\theta_{t}, \mathcal{M}_{i})$. Thus, the shared gradient $\bar{g}$ can be obtained as follows:

\begin{equation}
\begin{aligned}
     \bar{g} = \frac{\partial \mathcal{L}_{\text{shared}}^{old}(\theta^{t}, \mathcal{M}_{t-1}^{cor})}{\partial \theta_t} = \frac{1}{t-1} \sum_{i=1}^{t-1}\frac{\partial \mathcal{L}^{\old}(\theta_{t}, \mathcal{M}_{i})}{\partial \theta_t},
\end{aligned}
     \label{eq:shared_def}
\end{equation}
where the equality holds since partial derivatives are linear.

\noindent \textbf{Task-specific Gradients.} The task-specific gradients $\hat{g}_{i}$ are then obtained by subtracting the shared gradient from the gradient of each task,
\begin{equation} \label{eq:spec_def}
  \hat{g}_{i} = \frac{\partial  \mathcal{L}^{\text{old}}( \theta_{t}, \mathcal{M}_i) }{ \partial \theta_t} -  \bar{g}.
\end{equation}
We further construct the task-specific matrix $\hat{G}=[\hat{g}_1, \hat{g}_2,\cdots,\hat{g}_{t-1}]\in\mathbb{R}^{|\theta_{t}|\times (t-1)}$, where $|\cdot|$ represents cardinality. We denote the column space of $\hat{G}$ by $\hat{\mathcal{G}}$. Since the shared component is subtracted, the dimension of $\hat{\mathcal{G}}$ is at most $t-2$. Here we assume that the number of tasks is less than the number of parameters of the model, i.e., $N<|\theta_{t}|$, which generally holds for deep network. The schematic illustration of gradient decomposition is presented in Figure~\ref{fig:algo}(a).

\subsection{Gradient Optimization} \label{paragraph:GO}

Our algorithm aims to find a gradient $w$ that reduces the loss of the new task but does not increase the losses of any memory tasks. Inspired by GEM, if we assume the shared loss function is locally linear, the change of loss can be diagnosed by the sign of the inner product between its corresponding gradient $\bar{g}$ and the update $w$. The positive, zero, or negative inner product of $w$ and $\bar{g}$ indicates that the shared loss will decrease, preserve or increase respectively if we update the network by $-\eta w$, where $\eta$ is a small positive value. Similar consideration applies to the task-specific residuals by replacing $\bar{g}$ with $\hat{g}_i$. 

Based on the above observations, our algorithm requires that the gradient update $w$ should be close to $g$ by minimizing $\|w -g\|_{2}^{2}$ and will not increase the shared loss by constraining $\bar{g}^\top w \ge 0$, That is:
\begin{equation}
    \underset{w}{\min}\quad\frac{1}{2}||w-g||^2_2, \qquad \text{s.t.} \quad \bar{g}^{\top} w \ge 0.
\end{equation}

However, when considering the task-specific gradients $\hat{g}_{i}$, we find that the only solution to $\hat{g}_i^\top w \ge 0, \forall i < t$ is $\hat{g}_i^\top w = 0, \forall i < t$. According to Equ.~\eqref{eq:spec_def}, for any gradient update $w$, we have $\sum_{i=1}^{t-1} \hat{g}_{i}^{\top} w = 0$ because $\sum_{i=1}^{t-1} \hat{g}_{i}=\mathbf{0}$. It means that unless $\hat{g}_i^\top w = 0, \forall i < t$, any other $w$ would result in $\hat{g}_i^\top w>0$ for some $i$ while $\hat{g}_i^\top w<0$ for some other $i$, i.e., the losses of some tasks will increase but the losses of others will decrease. To better preserve the knowledge from each old task, the only choice is to require the gradient update $w$ orthogonal to the column space of the task-specific gradients, i.e.,
\begin{equation}
\hat{G}^\top w = \mathbf{0} 
\end{equation}

%--------------------------------------------------------------------------
For the ease of optimization, $\hat{G}^\top w = \mathbf{0}$ can be transformed according to the following lemma:

\begin{lemma}
    For a matrix $X\in\mathbb{R}^{N\times n}$, where $\rank(X)=r$ and $r\leq n<N$, there must exist $Y\in\mathbb{R}^{N\times r}$ and $A\in\mathbb{R}^{r\times n}$ that satisfy:
    \begin{equation}
        X=YA,
    \end{equation}
    where $Y^\top Y=\bf{I}$, $\rank(Y)=r$ and $\rank(A)=r$. Moreover, for a vector $v\in\mathbb{R}^N$, $X^\top v=\mathbf{0}$ if and only if $Y^\top v=\mathbf{0}$.
    \label{lemma:basis}
\end{lemma}
\begin{proof}
    See supplementary material.
\end{proof}

\begin{figure*}
\centering
\includegraphics[width=0.8\linewidth]{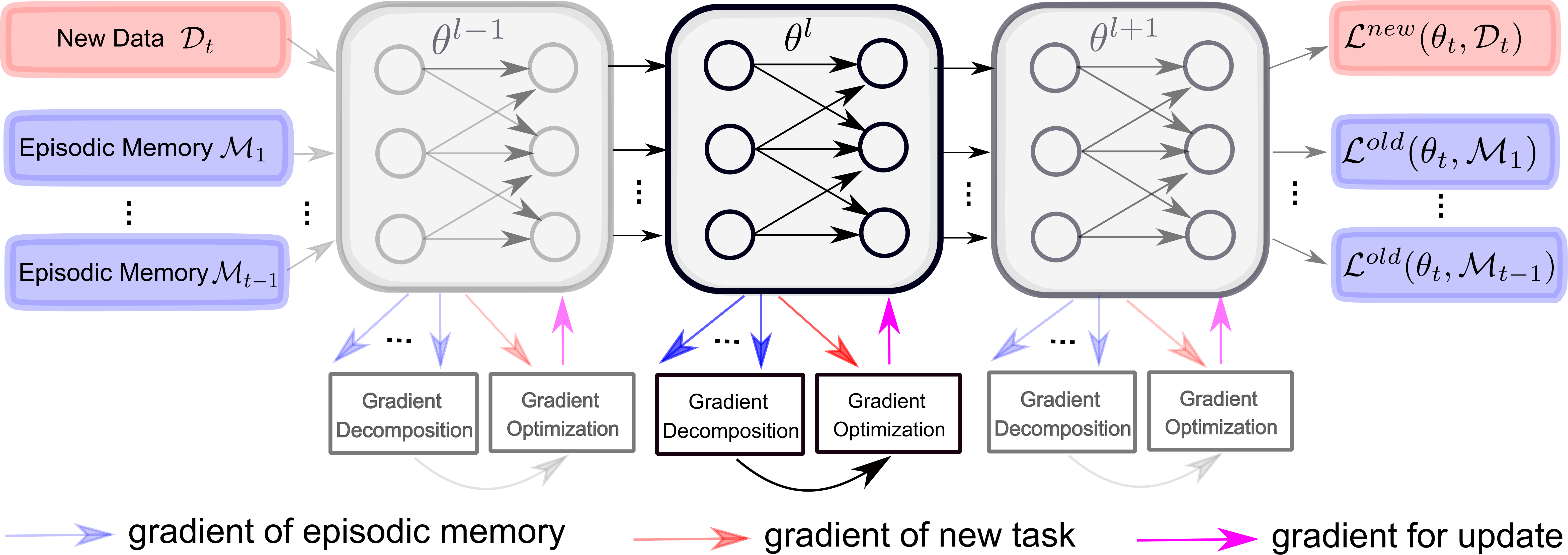}
\caption{\small{Illustration of layerwise gradient update strategy, which applies the proposed gradient decomposition and optimization to each layer independently.}}
\label{fig:layer_algo}
\vspace{-1em}
\end{figure*}

\noindent This lemma states that each column vector of an arbitrary matrix $X$ with rank $r$ can be represented by the linear combination of $r$ linearly independent column vectors. Therefore, the equality constraints  $\hat{G}^\top w = \mathbf{0}$  can be transferred into
\begin{equation}
  B^\top w = \mathbf{0} , 
\end{equation}
where $B=[b_1,\cdots,b_r]\in\mathbb{R}^{|\theta_t|\times r}$ is an orthogonal matrix that can be obtained by Gram-Schmidt process. The column space of $B$ equals to that of $\hat{G}$, i.e., $\hat{\mathcal{G}}$.

Based on the above analysis, the gradient update $w$  can be obtained by solving the following optimization problem:
\begin{equation} \label{eq:grad_optim_trans}
    \begin{split}
    \underset{w}{\min}\quad&\frac{1}{2}||w-g||^2_2, \\
    \text{s.t.} \quad &\bar{g}^{\top} w \ge 0, \\
    &B^\top w = \mathbf{0}
    \end{split}
\end{equation}
After transformation, we can solve the optimization problem in Equ.~\eqref{eq:grad_optim_trans} by Karush-Kuhn-Tucker condition~\cite{boyd2004convex}. The solution is as follows:
\begin{equation} \label{eq:new11}
w =
\left\{
\begin{array}{rcl}
      Pg, &  & \bar{g}^{\top}Pg \ge 0 \\
      Pg -  \frac{\bar{g}^{\top}Pg}{\bar{g}^{\top}P\bar{g}} P\bar{g},& &  \bar{g}^{\top}Pg  < 0
\end{array}
\right.   
\end{equation}
where $P = \mathbf{I} - BB^{\top}$, which is positive semidefinite. The schematic representation of gradient optimization is present in Figure~\ref{fig:algo}(b).

\noindent \textbf{PCA relaxation:} 
The equality constrains in the optimization problem in Equ. \eqref{eq:grad_optim_trans}, i.e., $B^\top w=\mathbf{0}$, enforce $w$ to be perpendicular to the whole residue space of old tasks. These constrains will become stronger as the number of old tasks increases, which may prevent $w$ from learning the new task. To address this problem, we propose a method to relax these constrains by replacing $B$ with its first few principal components, which can be obtained by principal component analysis. Although the above analysis defines the loss on the whole dataset, this analysis also applies to the loss on the data within each mini-batch, which replaces the whole data in Equ.~\eqref{eq:loss_def} by the data from a mini-batch. In each mini-batch, we update the parameters by $\theta_{t} \leftarrow \theta_{t} - \eta w$, where $\eta$ is the learning rate and $w$ is obtained from Equ.~\eqref{eq:new11}.

%------------------------------------------------------------------------------------------------------------
\subsection{Layerwise Gradient Update} \label{paragraph:LGU}

Traditional gradient-based CL algorithms, e.g. GEM, concatenate gradients from all layers together and the optimization is based on this concatenated gradient. The loss decrease of episodic memory $\mathcal{M}_i$ after updating parameters $\theta_t$ by $-\eta w$ is 
\begin{small}
\begin{equation*}
\Delta \mathcal{L}_1^{\old}(\theta_t, \mathcal{M}_i) \approx-\eta\sum_{l} \frac{\partial \mathcal{L}^{\old}(\theta_t, \mathcal{M}_i)}{\partial \theta_t} w,
\end{equation*}
\end{small}
where $w$ is obtained by Equ.~\eqref{eq:new11}. Considering $\frac{\partial \mathcal{L}^{\old}(\theta_t, \mathcal{M}_i)}{\partial \theta_t}$ is the summation of the shared gradient $\bar{g}$ and its task-specific gradient $\hat{g}$, the loss change of episodic memory $\mathcal{M}_i$ will be negative if $\bar{g}Pg \ge 0$ holds but will be zero if $\bar{g}Pg < 0$ after replacing $w$ with Equ~\eqref{eq:new11}. Specifically, 
\begin{small}
\begin{equation}
\Delta \mathcal{L}_1^{\old}(\theta_t, \mathcal{M}_i) = \left\{
\begin{array}{rcl}
      -\eta\bar{g}^\top Pg, &  & \bar{g}^{\top}Pg \ge 0 \\
      0, & &  \bar{g}^{\top}Pg  < 0
\end{array}
\right.   
\label{eq:concat}
\end{equation}
\end{small}
We notice that the condition $\bar{g}^{\top}Pg$ heavily depends on the elements with large magnitude in concatenated gradients as inner product is used. Since we have observed that the magnitude of gradients have a large variation among different layers from Figure~\ref{fig:magnitute_variation}, $\Delta \mathcal{L}_1^{\old}(\theta_t, \mathcal{M}_i)$ will be dominated by some layers where the magnitude of new gradient $g$, shared gradient $\bar{g}$ and task-specific gradients $\hat{G}$ is large.
\begin{figure} \label{fig:magnitude_dis}
    \centering
    \includegraphics[width=0.7\linewidth]{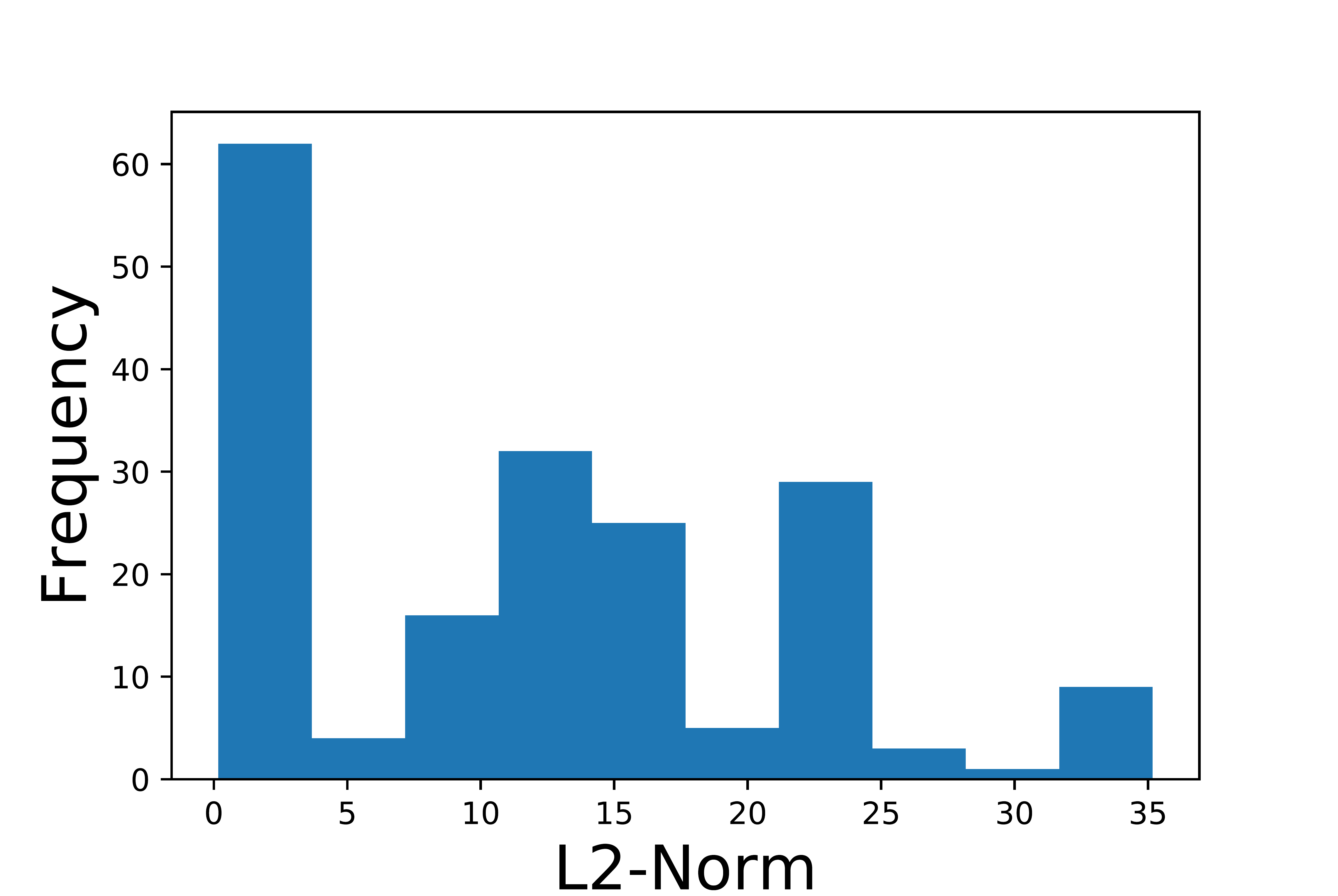}
    \caption{\begin{small}Histogram of the magnitude of gradients from different layers. We use L2-Norm to quantify the magnitude.\end{small} }
    \label{fig:magnitute_variation}
\end{figure}

For our proposed layerwise gradient update strategy, we replace concatenated $g, P$ in Equ.\eqref{eq:new11} with layer-specific gradients $g^{(l)}, P^{(l)}$, where $l$ is layer index.  Suppose the network is updated layerise, i.e., $\theta_t^{(l)}\leftarrow\theta_t^{(l)}-\eta w^{(l)}$, the loss change of episodic memory $\mathcal{M}_I$ is $-\eta\sum_{l=1}^L  \bar{g}^{(l)^{\top}} P^{(l)^\top}g^{(l)^\top}$, where $L$ is the number of layers. For layer $l$ that satisfies $\bar{g}^{(l)^\top} P^{(l)^\top} g^{(l)}\ge 0$, its contribution to $\Delta \mathcal{L}_2^{\old}$ is $-\eta \bar{g}^{(l)^\top}P^{(l)^\top}g^{(l)} \le 0$. For layer $l$ that satisfies $\bar{g}^{(l)^\top}P^{(l)^\top}g^{(l)} < 0$, its contribution to $\Delta \mathcal{L}_2^{\old}$ is zero. Therefore, the total loss change of episodic memory $\mathcal{M}_{i}$, $\Delta \mathcal{L}_2^{\old}$, can be presented as

\begin{equation} \label{eq:layerloss}
\Delta \mathcal{L}_2^{\old} = -\eta\sum_{l\in W^+}\bar{g}^{(l)^\top}P^{(l)^\top}g^{(l)} < 0
\end{equation}
where $\{W^+\}$ is the set of layer index whose condition $\bar{g}^{(l)^\top}P^{(l)^\top}g^{(l)}\ge 0$ is satisfied. We notice that since the gradient magnitude will not have a large variation within one layer, $\Delta \mathcal{L}_2^{\old} $ calculated layerwise is not dominated by a few elements with large magnitudes in $g^{(l)}$ and $P^{(l)}$.

Comparing Equ.~\eqref{eq:concat} with Equ.~\eqref{eq:layerloss}, the layerwise method has two advantages: first, instead of heavily depending on the gradients with large magnitude in the concatenating method, the layerwise method treats each layer equally by letting each layer determine its own condition. 

Second, as each layer determines its own condition, $\bar{g}^{(l)^\top}P^{(l)^\top}g^{(l)}\geq 0$ holds for some layers even though $\bar{g}^{\top}P^{\top}g<0$, which may reduce the loss more efficiently.

\subsection{Comparisons with GEM and A-GEM}
Since GEM and A-GEM are two closest algorithms to our method, we would make more explicit comparisons with them to clarify our innovations. 

\noindent \textbf{Similarity.} 
All three methods tackle ``\emph{catastrophic forgetting}'' by modifying the gradients of new task under the constraints related to the gradients of old tasks.

\noindent \textbf{Difference.} 
They are different in three aspects. (1) They differ by the constraints: GEM considers the old tasks to be independent by enforcing independent losses for old tasks. A-GEM only uses a global loss for all old tasks. As admitted in \cite{chaudhry2018efficient}, due to the use of single global loss, A-GEM may lead to the performance drop of some tasks. Our task-specific constraints can better exploit the inter-task dependency, as analyzed in Section \ref{paragraph:GO}. 
%GEM imposes multiple constraints for old tasks (one constraint per task) while A-GEM imposes only one constraint for all old tasks. However, both GEM and A-GEM ignore the inter-task relationships and thus no gradient decomposition is implemented. Our constraints are based on a shared gradient and task-specific gradients, which can better exploit the inter-task consistency.
(2) They differ in the efficiency for solving the optimization problem. GEM has to solve QP online, while A-GEM and ours have closed-form solutions, which is computationally efficient. (3) Our proposed Layerwise gradient update and PCA relaxation are not investigated in GEM or A-GEM. The advantage and motivation of our contribution on layerwise gradient update are introduced in Section~\ref{paragraph:LGU}.

\section{Experiments}
\subsection{Setup}
\noindent \textbf{Datasets.} We employ four standard benchmark datasets to evaluate the proposed continual learning framework, including MNIST Permutation, the split CIFAR10, the split CIFAR100 and the split \emph{tiny}ImageNet datasets. The MNIST Permutation is a synthetic dataset based on MNIST~\cite{lecun1998mnist}, where all pixels of an image are permuted differently but coherently in each task. The split CIFAR10/100 dataset is generated from CIFAR10/100~\cite{krizhevsky2009learning}. The split \emph{tiny}ImageNet is derived from \emph{tiny}ImageNet~\cite{le2015tiny}. These datasets equally divide their target classes into multiple subsets, where each subset corresponds to individual task. We considered $T=20$ tasks for MNIST Permutation, Split CIFAR100 and Split \emph{tiny}ImageNet, and $T=5$ for Split CIFAR 10. For each dataset, each task contained samples from a disjoint subset of classes, except that on MNIST two consecutive tasks contained disjoint samples from the same class. The evaluation was performed on the test partition of each dataset.

\noindent \textbf{Implementation Details}.
We trained our models following the description in GEM~\cite{lopez2017gradient}. On the MNIST tasks, we used fully-connected neural networks with two hidden layers of 100 ReLU units. On the CIFAR10/CIFAR100 tasks, we used a smaller version of ResNet18~\cite{he2016deep}, which has three times less number of feature maps for all layers than the standard ResNet18. On the \emph{tiny}ImageNet, we empolyed the standard ResNet18. We trained all networks using the plain stochastic gradient descent optimizer with mini-batches of 10 samples for the new task on MNIST Permutation, Split Cifar10/Cifar100 and 20 samples on \emph{tiny}ImageNet. The batch size for each episodic memory task is 20. We followed the same one/three epoch(s) settings as GEM, where the samples in the new task were only trained once/three times but samples in the memory could be trained for several times. The learning rate was 0.1 for all datasets. The pseudo code of training is presented in Algorithm~\ref{Algor:train}.

\begin{algorithm}
\caption{Training Procedures of Step $t$}
\label{Algor:train}
\footnotesize
\begin{algorithmic}[1]
\Require
    \Statex $\{\mathcal{M}_{i}\}_i^{t-1}$: episodic memory of old tasks
    \Statex $\mathcal{D}_t$: new data of current step $t$
    \Statex $f$: network architecture
    \Statex $\theta_{t-1}$: network paramters trained by last step $t-1$
    \Statex $bs_\new$: batch size of the new task
    \Statex $bs_\old$: batch size of the old tasks
    \Statex $\eta$: learning rate
\Ensure
    \Statex $\mathcal{M}_t$: episodic memory of current step $t$
    \Statex $\theta_{t}$: network parameters trained by current step $t$
\State Initialization: $\theta_t=\theta_{t-1}$
\For{$epoch=1:epoch\_stop$}
\For{$iter=1:iter\_stop$}
    \State Sample a batch $s_\new$ of size $bs_\new$ from $\mathcal{D}_t$
    \State $g \leftarrow \nabla \mathcal{L}(f(x;\theta_t), y)$ for $(x,y)\in s_\new$
    \For{$i=1:t-1$}
        \State Sample a batch $s_i$ of size $bs_\old$ from $\mathcal{M}_i$
        \State $g^i \leftarrow \nabla \mathcal{L}(f(x;\theta_t), y)$ for $(x,y)\in s_i$ 
    \EndFor
    \State Compute shared component $\bar{g}$ by Equ.~\eqref{eq:shared_def}
    \State Compute specific components $\hat{G}$ by Equ.~\eqref{eq:spec_def}
    \If{PCA Relaxation}
        \State Compute the principal components $B$ of $\hat{G}$ by PCA
    \Else
        \State Compute the orthogonal basis $B$ of $\hat{G}$ by Schmidt orthogonalization
    \EndIf
    \State Compute the gradient for update $w$ by Equ.~\eqref{eq:new11}
    \State Update $\theta_t$: $\theta_t\leftarrow\theta_t-\eta w$
\EndFor
\EndFor
\State $\mathcal{M}_t \leftarrow$ a subset of $\mathcal{D}_t$
\end{algorithmic}
\end{algorithm}

\noindent \textbf{Evaluation Metrics.} Following~\cite{chaudhry2018riemannian,lopez2017gradient,chaudhry2018efficient}, we measure the performance by mean classification accuracy (ACC) and  Backward Transfer (BWT). BWT is defined as the change of average accuracy for old tasks after learning a new task. A positive value of BWT means that learning new tasks can benefit old tasks, while a negative value indicates that learning new task degrades the performance of old tasks.

\subsection{Empirical analysis}
To investigate the contribution of new components in our proposed method, we incrementally evaluate each of them on MNIST Permutation, Split CIFAR10, Split CIFAR100 and Split \emph{tiny}Imagenet. Denote \textbf{SCC} as the \textbf{S}hared \textbf{C}omponent \textbf{C}onstraint, \textbf{TSCC} as \textbf{T}ask-\textbf{S}pecific \textbf{C}ompoent \textbf{C}onstraint, \textbf{LGU} as \textbf{L}ayer\-wise \textbf{G}radient \textbf{U}pdate, and \textbf{PCA} as \textbf{P}rincipal \textbf{C}ompoent \textbf{A}nalysis. We choose a single predictor fine-tuned across all tasks as baseline. Seven variants are then constructed on top of baseline: (a) baseline; (b) baseline+SCC; (c) baseline+SCC+LGU; (d) baseline+SCC+TSCC; (e) baseline+SCC+TSCC+PCA; (f) baseline+SCC+LGU+TSCC; (g) baseline+SCC+LGU+TSCC+PCA. Table~\ref{tab:results_components} presents the performance of all variants. All reported results are tested on the task-incremental setting.

% \begin{table}[t]
% \centering \footnotesize
% \begin{tabular}{c||c|c|c|c|c|c|c}
% \hline
% Methods                 & (a)        & (b)        & (c)        & (d)           & (e)     & (f)      & (g)\\ \hline
%  SCC         &  &\checkmark&\checkmark&\checkmark   &\checkmark &  &  \\
%                              TSCC    &          &    &          &\checkmark   &\checkmark  &  &\\
%                              LGU         &          &          &\checkmark&\checkmark   &\checkmark &  &  \\
%                               PCA          &          &          &          &             &\checkmark &   &  \\  \hline
%  MNIST        &57.4      &76.4      &81.9      &\textbf{82.9}&\textbf{82.9}  & & \\
%                              CIFAR10      &58.7      &75.8      &76.8      &77.3         &\textbf{79.0}  & & \\
%                              CIFAR100     &54.7      &65.3      &67.8      &69.1         &\textbf{69.3} & & \\
%                              \footnotesize \emph{ting}Imagenet &21.8      &33.5      &36.4      &37.6         &\textbf{38.3} & &\\ \hline
% \end{tabular}
% \caption{
% \small Empirical analysis of the proposed components on MNIST Permutation, Split CIFAR10, Split CIFAR100 and Split \emph{tiny}Imagenet datasets in one epoch setting. (a-f) denotes different methods.} 
% \label{tab:results_components}  
% \end{table}

\begin{table}[]
\centering \footnotesize
\renewcommand\arraystretch{1.1}    
\begin{tabular}{c||ccccccc}
\hline
Method       & (a)  & (b)       & (c)       & (d)       & (e)       & (f)       & (g)       \\ \hline
SCC          &      & \checkmark & \checkmark & \checkmark & \checkmark & \checkmark & \checkmark \\ \hline
TSCC         &      &           &           & \checkmark & \checkmark & \checkmark & \checkmark \\ \hline
LGU          &      &           & \checkmark &           &           & \checkmark & \checkmark \\ \hline
PCA          &      &           &           &           & \checkmark &           & \checkmark \\ \hline \hline
MNIST        & 57.4 & 76.4      & 82.0      & 81.9      & 81.9      & 82.9      & 82.9      \\ \hline
CIFAR10      & 58.7 & 75.8      & 76.2      & 76.8      & 77.5      & 77.3      & 79.0      \\ \hline
CIFAR100     & 54.7 & 65.3      & 67.3      & 67.8      & 68.2      & 69.1      & 69.3      \\ \hline
\emph{tiny}Imagenet & 21.8 & 33.5      & 36.3      & 36.4      & 38.1      & 37.6      & 38.3      \\ \hline
\end{tabular}
\caption{
\small Empirical analysis of the proposed components on MNIST Permutation, Split CIFAR10, Split CIFAR100 and Split \emph{tiny}Imagenet datasets in one epoch setting. (a-f) denotes different methods.} 
\label{tab:results_components}  
\end{table}

\noindent \textbf{Effectiveness of Gradient Decomposition.} 
We compare Methods (a,b,d) in Table~\ref{tab:results_components} to illustrate to what extent the shared gradient constraint and task-specific gradients constraints contribute to the performance. Split CIFAR100 and Split \emph{tiny}Imagenet are taken as examples for analysis. We observe that the shared gradient constraint plays a significant role in continual learning task as only applying the shared gradient constraint improves the performance of the model from $54.7\%$ to $65.3\%$ on Split CIFAR100 and from $21.8\%$ to $33.5\%$ on Split \emph{tiny}ImageNet at the final continual step. The task-specific constraints can bring additional gain for the model. It improves the mean accuracy from $65.3\%$ to $67.8\%$ on Split CIFAR100 and from $33.5\%$ to $36.4\%$ on Split \emph{tiny}ImageNet at the final continual step. The more enhancement of performance on both two datasets verifies that the shared gradient constraint contains most of the knowledge for old tasks and it is the most important constraint for continual learning. 

\noindent \textbf{Effectiveness of PCA Relaxation.} To show PCA relaxation can empower our method a better trade-off between learning and memorizing, we evaluate BWT and ACC under different value of K on both Split CIFAR100 and Split \emph{tiny}ImageNet datasets. Here we denote K as the rank of task-specific matrix after relaxation.
\begin{figure}{}
\centering
\includegraphics[width=0.9\linewidth]{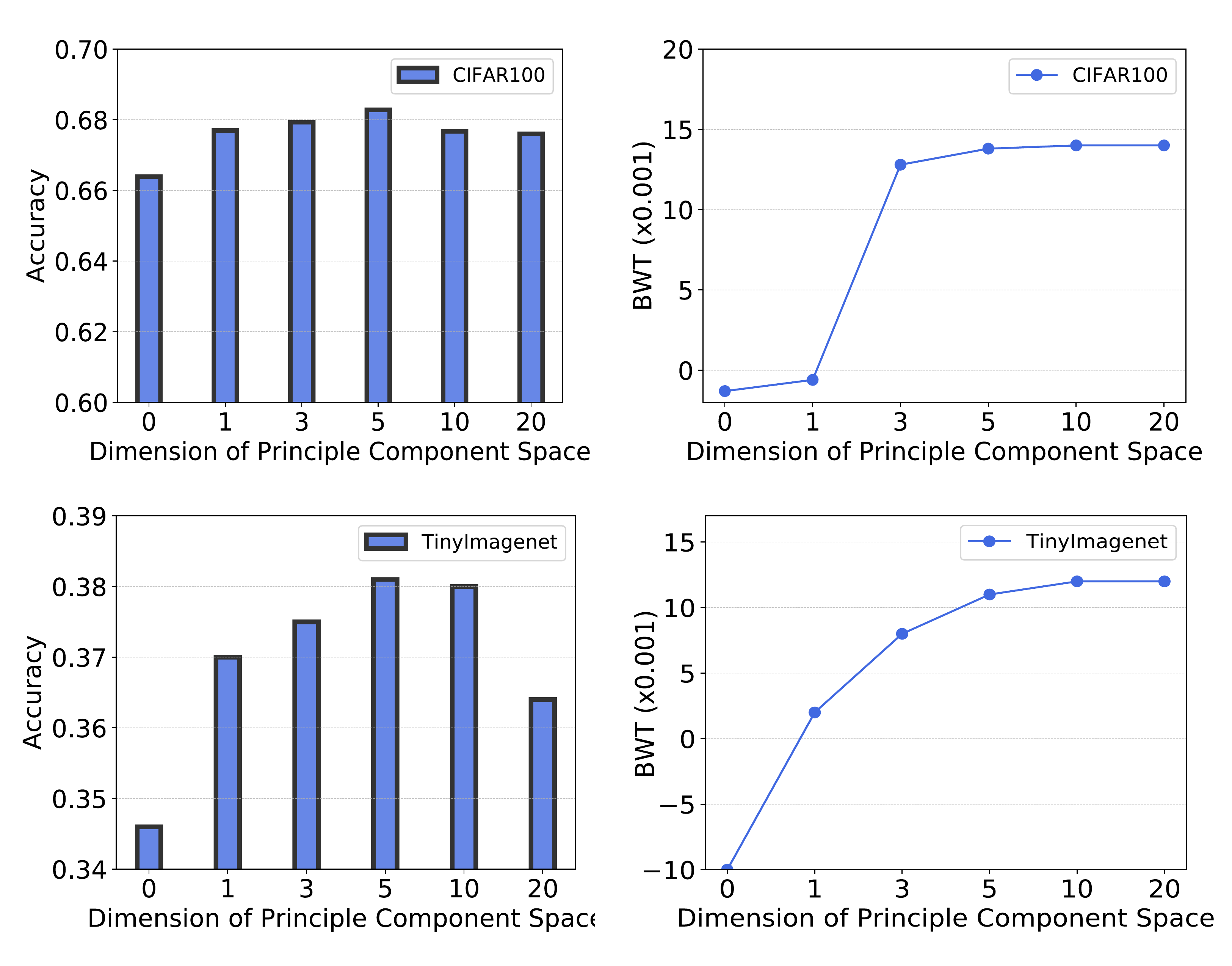}
\caption{\small{Average Accuracy and Backward Transfer with different dimension of principal component space. Top row: ACC and BWT on CIFAR 100. Bottom row: ACC and BWT on \emph{tiny}ImageNet.}}
\label{fig:exp_AB2}
\vspace{-1em}
\end{figure}
Figure~\ref{fig:exp_AB2} shows that as K increases, ACC increases initially but decreases afterwards, while BWT keeps increasing. The continued increase of BWT is expected, as more knowledge from the old tasks is preserved as K increases. However, ACC decreases after K $=5$, indicating that although the performance on the old tasks benefits from a large K, the ability to learn new task would be degraded if K is too large.

\noindent \textbf{Effectiveness of Layer-wise Gradient Update.}
We propose the layer-wise gradient update strategy where the constraint is imposed by individual network layer. We compare Methods (b,c) in Table~\ref{tab:results_components} to manifest its general effectiveness. The results show that with layer-wise gradient update strategy, Average Accuracy ACC is higher than those by the method without layer-wise gradient update strategy. In addition, from method (b-e) in Table~\ref{tab:results_components}, we can conclude that Gradient decomposition (SCC, TSCC) and Layerwise Gradient Update (LGU) are orthogonal and cumulative. To further illustrate the effectiveness of layerwise gradient update strategy, we leverage layerwise gradient update on top of GEM, A-GEM and our proposed method (d) and present the results in Table~\ref{tab:layerwise_compare}. By comparing GEM, A-GEM, Method (d) with their layerwise gradient update counterparts, we observe that both ACC and Backward Transfer (BWT) of algorithms with layerwise gradient update are systematically higher than those without the strategy. The results empirically prove the general effectiveness of layerwise gradient update on GEM Families.

\begin{table}[t]
\centering
\footnotesize
\begin{tabular}{c|c|cc|cc|cc}
\hline
\multirow{2}{*}{\scriptsize Datasets}     & \multirow{2}{*}{\scriptsize LGU} & \multicolumn{2}{c|}{GEM}       & \multicolumn{2}{c|}{A-GEM}      & \multicolumn{2}{c}{Method (d)} \\ \cline{3-8} 
                              &                           & \scriptsize ACC           & \scriptsize BWT            & \scriptsize ACC           & \scriptsize BWT            & \scriptsize ACC            & \scriptsize BWT            \\ \hline \hline
\multirow{2}{*}{\scriptsize MNIST}        &  x                        & 81.9          & 0.017          & 76.4          & 0.005         & 81.9           & 0.017          \\
                              & \checkmark                & \textbf{82.7} & \textbf{0.045} & \textbf{77.1} & \textbf{-0.008} & \textbf{82.9}  & \textbf{0.045} \\ \hline
\multirow{2}{*}{\scriptsize CIFAR10}      &  x                        & 76.8          & 0.014          & 75.8          & 0.007          & 76.8           & 0.015          \\
                              & \checkmark                & \textbf{77.4} & \textbf{0.035} & \textbf{76.5} & \textbf{0.015} & \textbf{77.3}  & \textbf{0.034} \\ \hline
\multirow{2}{*}{\scriptsize CIFAR100}     &  x                        & 65.8          & -0.005         & 65.3          & -0.005         & 67.8           & -0.001         \\
                              & \checkmark                & \textbf{67.9} & \textbf{0.003} & \textbf{66.8} & \textbf{0.002 }           & \textbf{69.1} & \textbf{0.006}          \\ \hline
\multirow{2}{*}{\scriptsize \emph{tiny}Imagenet} & x                         & 34.2          & -0.001         & 33.5          & -0.003         & 36.4           & 0.008          \\
                              & \checkmark                & \textbf{35.8} & \textbf{0.006} & \textbf{35.6} & \textbf{0.005} & \textbf{37.6}  & \textbf{0.018} \\ \hline
\end{tabular}
\caption{
\small Empirical Analysis of layerwise gradient update on GEM, A-GEM and Method (d).} 
\label{tab:layerwise_compare}
\end{table}

% \noindent \textbf{Computational complexity reduction by LGU.} The main cost of our algorithm stems from Equ.~\eqref{eq:new11}. We take (d) and (f) in Table~\ref{tab:results_components} to analyze the time complexity of layerwise update. Denote $T$ as the total task number and $N$ as the number of model parameters. The total time cost of (d) is $O((T+1)N^2+(2T^2+1)N)$: $O(2T^2N)$ for calculating $B$ by Schmidt process, $O(TN^2)$ for calculating $P$, and $O(N^2+N)$ for getting $w$. Similarly, the total time cost of (f) is roughly $O((T+1)N^2/L+(2T^2+1)N)$, where $L$ is the number of network layers. Since $L > 1$, the time cost of (f) is possibly less than that of (d). By implementing 3-epochs setting experiments on MNIST, we found the running time of (d) and (f) for 20 tasks was 286.3s and 253.4s respectively, which empirically verified our analysis.

\begin{table*}[t]  
  \centering \small
 \resizebox{1\hsize}{21mm}{
 \setlength{\tabcolsep}{2.7mm}
 {
\begin{tabular}{c|cc|cc|cc|cc}
\hline
\multirow{2}{*}{Method} & \multicolumn{2}{c|}{MNIST Permutation} & \multicolumn{2}{c|}{Split CIFAR10} & \multicolumn{2}{c|}{Split CIFAR100} & \multicolumn{2}{c}{Split \emph{tiny}ImageNet} \\ \cline{2-9} 
                        & 1 epoch           & 3 epochs           & 1 epoch      & 3 epochs      & 1 epoch       & 3 epochs      & 1 epoch   & 3 epochs              \\ \hline
Single~\cite{robbins1951stochastic}                  & $57.4\pm0.7$              & $53.0\pm0.6$               & $57.7 \pm 0.3$         & $68.4 \pm 0.8$          & $54.7\pm1.0$          & $55.4\pm0.9$          & $25.7\pm0.8$      & $28.7\pm0.5$                     \\
Independent~\cite{lopez2017gradient}             & $37.0\pm0.6$              & $61.3\pm0.4$               & $43.3\pm0.8$         & $71.9\pm0.7$          & $43.3\pm0.3$          & $53.5\pm0.5$          & $26.9\pm0.6$      &  $35.9\pm0.6$                     \\
Multimodel~\cite{lopez2017gradient}              & $71.9\pm0.5$              & $59.4\pm0.5$               & -            & -             & -             & -             & -         & -                     \\
EWC~\cite{kirkpatrick2017overcoming}                     & $57.5\pm0.6$              & $61.9\pm0.5$               & $59.8\pm0.4$         & $67.8\pm0.5$          & $48.3\pm0.6$          & $56.9\pm0.6$          & $22.2\pm0.4$      & $25. 4\pm0.3$  \\
iCARL~\cite{rebuffi2017icarl}                   &  -                &  -                 & $63.9\pm0.5$         & $66.5\pm 0.6$          & $55.7\pm0.8$          & $55.8\pm0.4$          & $26.5\pm0.6$      & $28.9\pm0.4$                      \\
GEM~\cite{lopez2017gradient}                     & $81.9\pm0.4$              & $84.1\pm0.4$               & $76.8\pm0.3$         & $80.8\pm0.3$          & $65.8\pm0.5$          & $69.6\pm0.4$          & $34.2\pm0.5$      & $37.3\pm 0.4$                       \\
A-GEM~\cite{chaudhry2018efficient}                   & $76.4\pm 0.3$              & $82.9\pm0.5$               & $75.8\pm0.6$         & $81.0\pm0.4$          & $65.3\pm0.4$          & $69.1\pm0.5$          & $33.5\pm0.3$      & $36.9\pm0.3$                      \\
S-GEM~\cite{chaudhry2018efficient}                   & $76.3\pm0.4$              & $82.3\pm0.4$               & $75.9\pm0.3$         & $81.5\pm0.3$          & $64.2\pm0.3$          & $68.8\pm0.2$          & $33.8\pm0.4$      & $36.1\pm0.3$    
      \\ 
tinyER\dag~\cite{chaudhry2019tiny}                   & $82.5\pm0.8$              & -               & $77.5\pm0.5$         & -          & $68.3\pm0.7$          & -          & $36.5\pm0.3$      &     -
      \\
GDumb\dag~\cite{prabhu2020gdumb}                   & $73.4\pm0.4$              & -               & $74.2\pm0.3$         &  -         & $60.3\pm0.9$          & -          & $32.1\pm0.3$      &     -
      \\  \hline
Ours                    & $\mathbf{82.9\pm 0.3}$      & $\mathbf{84.3\pm 0.3}$      & $\mathbf{79.0 \pm 0.4}$& $\mathbf{81.6\pm 0.5}$         &$\mathbf{69.3\pm 0.4}$  &$\mathbf{71.0\pm0.3}$  & $\mathbf{38.3\pm0.4}$ & $\mathbf{39.5\pm0.3}$                     \\ \hline 
\end{tabular}  }}
    \caption{
    \small Comparison of our proposed method with state-of-the-art methods on MNIST Permutation, Split CIFAR10, Split CIFAR100 and Split \emph{tiny}ImageNet in Online Task Incremental setting. All the experiments are conducted in three runs. \dag indicates that we follow the training schedule of Online Task Incremental Disjoint setting in ~\cite{chaudhry2019tiny} and ~\cite{prabhu2020gdumb} respectively, which means every new sample can only be trained once.} 
     \label{tab:state_of_the_art}  
\vspace{-1em}
\end{table*}

\subsection{Main results}
Our main results except MNIST are tested on the online task-incremental setting while MNIST is tested on the data-incremental setting.
We compare our proposed method with several existing and state-of-the-art methods, including Single~\cite{robbins1951stochastic}, Independent\footnote{For fair comparison, we follow the setting in GEM. In GEM, the channels of each feature map are divided by the number of tasks. As such, the model size in Independent method equals to the model size in other methods.}~\cite{lopez2017gradient}, Multimodel~\cite{lopez2017gradient}, EWC~\cite{kirkpatrick2017overcoming}, iCARL~\cite{rebuffi2017icarl}, GEM~\cite{lopez2017gradient}, A-GEM~\cite{chaudhry2018efficient} and S-GEM~\cite{chaudhry2018efficient}. For fair comparison, the size of each episodic memory is also 256 for GEM, A-GEM and S-GEM. The results are reported in Tab.~\ref{tab:state_of_the_art}. Our method achieves state-of-the-art results on all four datasets in both one epoch and three epochs settings.

\noindent \textbf{MNIST Permutation.} 
Our method achieves $82.9\%$ for one epoch setting and $84.3\%$ for three epochs setting, which outperforms the competitive GEM~\cite{lopez2017gradient} by $1.0\%$ and $0.2\%$, respectively. We only marginally outperforms GEM at the three epochs setting because the network has been saturated.

\noindent \textbf{Split CIFAR10.} On Split CIFAR10, our method significantly outperforms GEM by $2.2\%$ in one epoch setting. In three epochs setting, our method achieves similar results with S-GEM where we only improve by $0.1\%$ on Split CIFAR10, which is within statistical error. The marginal improvement can be explained by that the Split Cifar10 is oversimple and various medels can touch the upper bound accuracy of each task.

\noindent \textbf{Split CIFAR100.} On Split CIFAR100, we improve the ACC from $65.8\%$ to $69.3\%$ by $3.5\%$ in one epoch setting. In three epochs setting, we improve ACC from $69.6\%$ to $71.0\%$ by $1.4\%$.

\noindent \textbf{Split \emph{tiny}ImageNet.}
As illustrated in Tab.~\ref{tab:state_of_the_art}, we significantly improve the benchmark by $3.9\%$, which is from $34.2\%$ to $38.1\%$ in one epoch setting. In three epoches setting, our method also outperforms GEM by $2.6\%$.

\subsection{Results for the class-incremental setting}
Different from task-incremental continual learning setting mainly discussed in this paper, class-incremental continual learning setting do not provide task identity in the test data~\cite{hsu2018re,van2019three}. In such setting, task boundaries naturally formed in the continual training set ought to be broken because we need classify classes from different continual tasks. In order to break the task boundary, we first stack all episodic memories as one replay buffer. When training, we randomly split the replay buffer into $N$ sub-replay buffers. The gradient decomposition is implemented on such $N$ sub-replay buffers instead of episodic memories. We assess our method on Split CIFAR100 dataset for consistency with other approaches. We employ Resnet34 as our backbone and train the network for 200 epochs for each continual step. The initial learning rate is 0.1 and decays when the epoch equals 80, 160. The size of replay buffer is restricted to be 2000 images and the memory update strategy is the same as that in ~\cite{rebuffi2017icarl,castro2018end} for fair comparison. 

We compare our approach with other competing methods, including LWF~\cite{li2017learning}, iCarl~\cite{rebuffi2017icarl}, DR~\cite{hou2018lifelong}, End2End~\cite{castro2018end}, LUCR\footnote{LUCR is short for Learning a Unified Classifier via Rebalancing}~\cite{hou2019learning}, GD~\cite{lee2019overcoming}, MUC~\cite{fini2020online} and TPL~\cite{tao2020topology}. We fix $N=5$ when implementing the experiments. The comparisons with other competing methods are presented in Table~\ref{tab:state_of_the_art_class}. As illustrated in Table~\ref{tab:state_of_the_art_class}, our method outperforms GEM by $3\%$ which shows the consistency improvement with results in task-incremental scenario. We also achieve the state-of-the-art results which are the same MUC~\cite{fini2020online}.

\begin{table}[t]
\centering
\scriptsize
\begin{tabular}{|c|c|c||c|c|c|}
\hline
Method  & Ref     & Results & Method & Ref     & Results \\ \hline \hline
LwF$^{\rm a}$   & PAMI'~17 & 58.4    & GD     & ICCV'~19 & 62.1    \\ \hline
iCarl   & CVPR'~17 & 58.7    & A-GEM$^{\rm b}$  & ICLR'~19 & 60.43   \\ \hline
DR$^{\rm a}$     & ECCV'~18 & 59.1    & S-GEM$^{\rm b}$  & ICLR'~19 & 59.98   \\ \hline
End2End$^{\rm a}$ & ECCV'~18 & 60.2    & GEM$^{\rm b}$    & NIPS'~17 & 61.9    \\ \hline
LUCR     & CVPR'~19 & 61.2    & TPL   & ECCV'20        & \textbf{65.3}    \\ \hline
MUC & ECCV'~20 & 64.7 & Ours   & -        & \textbf{65.3} \\ \hline
\end{tabular}
\scriptsize{$^{\rm a}$ LwF and DR were adapted for class-incremental setting. Results of LwF, DR and End2End were reported in \cite{lee2019overcoming}.}\\
%LwF and DR are adapted for class-incremental setting, which are originally evaluated for task-incremental setting. Results are cited from \cite{lee2019overcoming}.}\\
{$^{\rm b}$ GEM, A-GEM, S-GEM are re-implemented by ourselves.}
\vspace{-1em}
\caption{
\small Comparison of our proposed method with state-of-the-art methods on Split CIFAR100 in the class-incremental setting.} 
\vspace{-1em}
\label{tab:state_of_the_art_class} 
\end{table}

\section{Conclusion}
In this work, we presented a novel continual learning framework including gradient decomposition, gradient optimization and layerwise gradient update. The gradients of episodic memory are decomposed into the shared gradient and task-specific gradients. Our framework encourages the consistency between the gradient for update and the shared gradient, and enforces the gradient for update orthogonal to the space spanned by task-specific gradients. The former keeps the common knowledge,  while the latter can be relaxed by PCA to preserve the task-specific knowledge. We observe that optimizing the gradient update layerwise can further help the model remember the old tasks.  Our method significantly outperforms current state-of-the-art on extensive benchmarks.

\section{Acknowledgement}
Wanli Ouyang was supported by the SenseTime, Australian Research Council Grant DP200103223, and Australian Medical Research Future Fund MRFAI000085.

{\small
\bibliographystyle{ieee}
\bibliography{egbib}
}

\end{document}

%%%%%%%%% TITLE
\title{Supplementary Material of Layerwise Optimization by Gradient Decomposition for Continual Learning}
\renewcommand{\thefootnote}{\fnsymbol{footnote}}
\author{Shixiang Tang$^1$\footnotemark[2] \qquad Dapeng Chen$^3$ \qquad Jinguo Zhu$^2$ \qquad Shijie Yu$^4$ \qquad Wanli Ouyang$^1$ \\
$^1$The University of Sydney, SenseTime Computer Vision Group, Australia \qquad $^2$Xi'an Jiaotong University \\
$^3$Sensetime Group Limited, Hong Kong \qquad $^4$Shenzhen Institutes of Advanced Technology, CAS \\  
{\tt\small tangshixiang@sensetime.com \quad  dapengchenxjtu@yahoo.com \quad  wanli.ouyang@sydney.edu.au }
}
% \and

\maketitle
\thispagestyle{empty}
\footnotetext[2]{This work was done when Shixiang Tang was an intern at SenseTime.}

\section{Proof of Lemma 1}
\begin{proof}
    Following singular value decomposition (SVD)~\cite{boyd2004convex}, $X=U\Sigma V^\top$, where $U\in\mathbb{R}^{N\times N}$ and $V\in\mathbb{R}^{n\times n}$ are unitary matrices, $\Sigma\in\mathbb{R}^{N\times n}$ is rectangular diagonal matrix with $r$ non-zero diagonal entries. That is, $\Sigma$ has $n-r$ zero column vectors. Without loss of generality, we assume that these zero column vectors are the last $n-r$ column vectors of $\Sigma$. Thus, $X=YA$, where $Y\in\mathbb{R}^{N\times r}$ is formed by the first $r$ column vectors of $U$, and $A\in\mathbb{R}^{r\times n}$ is formed by the first $r$ row vectors of $\Sigma V^\top$. The resulting $Y$ and $A$ satisfy that $Y^\top Y=\mathrm{I}$ and $\rank(A)=r$.
    
    As $\rank(A)=r$, $AA^\top\in\mathbb{R}^{r\times r}$ is full rank. Therefore, for any $v\in\mathbb{R}^N$, $v^\top X=v^\top YA=\mathbf{0}$ implies $v^\top Y=\mathbf{0}$. On the other hand, $v^\top Y=\mathbf{0}$ implies $v^\top YA=v^\top X=\mathbf{0}$, which completes the proof.
\end{proof}

\section{Solution to Equ. (10)}
Since the objective $||w-g||^2_2$ is quadratic with positive definite second-order coefficient and all the constrains are affine, the optimization in Equ.(10) of the main manuscript is convex. Therefore, Karush–Kuhn–Tucker (KKT) conditions~\cite{boyd2004convex} applies.

The Lagrangian function of the original constraint optimization problem can be defined as follows:
\begin{equation}
    L(w,\mu,\lambda)=\frac{1}{2}||w-g||^2_2-\mu \bar{g}^{\top} w + \lambda^\top B^\top w,
\end{equation}
where $\mu\geq0$. Following KKT conditions, the solution to the original optimization problem should satisfy:
\begin{equation}
    \begin{aligned}
    \nabla_w L(w,\mu,\lambda)&=\mathbf{0}\\
    B^\top w&=\mathbf{0}\\
    -\bar{g}^{\top} w&\leq 0\\
    \mu & \geq 0\\
    -\mu\bar{g}^{\top}w&=0
    \end{aligned}
\end{equation}
Solving these equations, we have
\begin{equation}
w =
\left\{
\begin{array}{rcl}
      Pg, &  & \bar{g}^{\top}Pg \ge 0 \\
      Pg -  \frac{\bar{g}^{\top}Pg}{\bar{g}^{\top}P\bar{g}} P\bar{g},& &  \bar{g}^{\top}Pg  < 0
\end{array}
\right.   
\end{equation}
where $P = \mathbf{I} - BB^{\top}$.

\section{Comparison with other Relaxation Methods}

In this section, we first illustrate the effectiveness of PCA relaxation compared with other relaxation methods. We denote K as the rank of task-specific matrix after relaxation.
\begin{figure}[h]
\centering
\includegraphics[width=\linewidth]{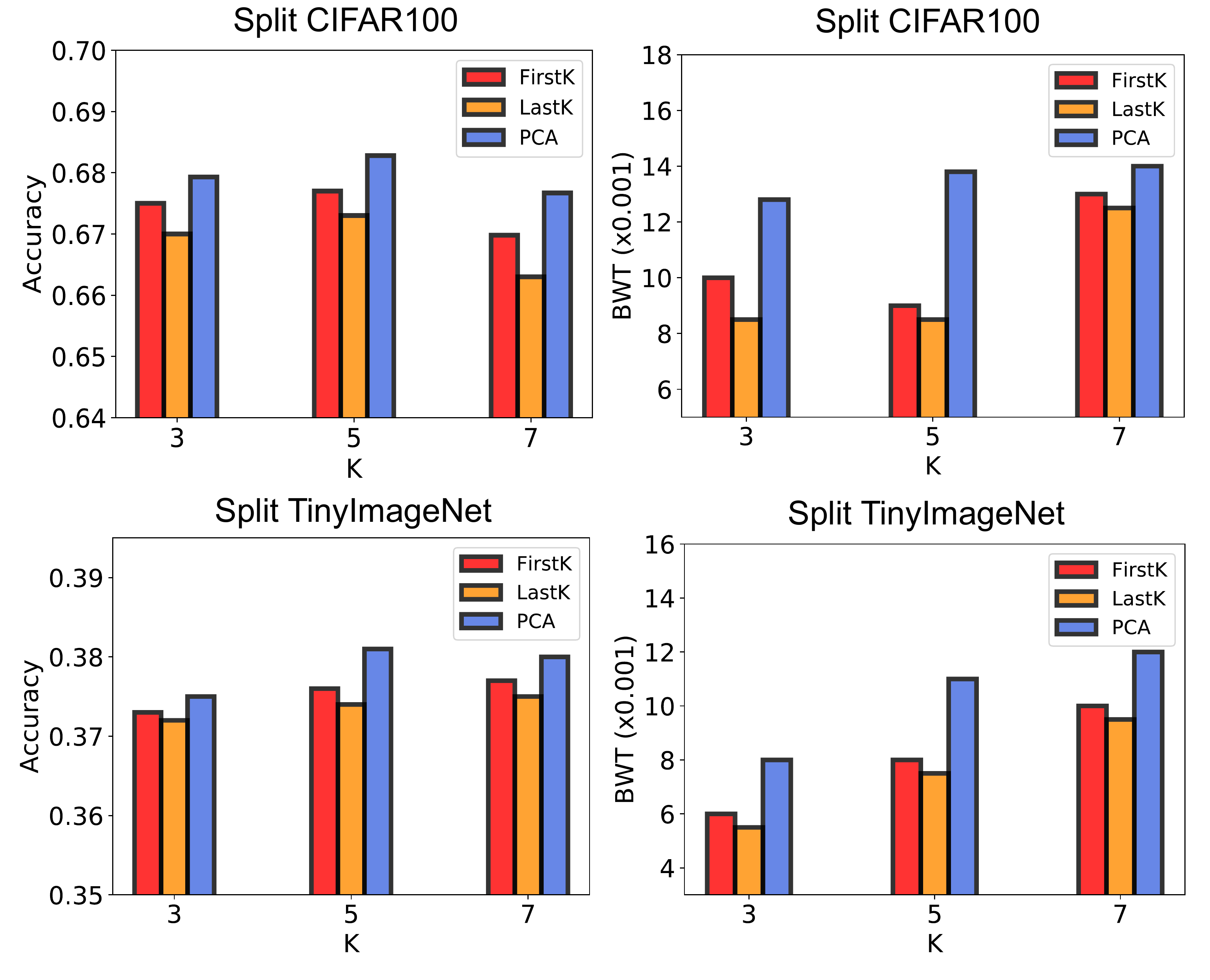}
\caption{\small{Average Accuracy and Backward Transfer with different relaxation methods. FirstK: preserving the first K elements of task-specific gradients matrix $\hat{G}$. LastK: preserving the last K elements of task-specific gradients matrix $\hat{G}$.}}
\label{fig:exp_AB3}
\end{figure}
In order to show that the PCA relaxation can preserve the important information in the task-specific constraints, we compare PCA Relaxation with other two methods. \textbf{FirstK} is to relax the task-specific matrix by preserving its first K elements and \textbf{LastK} is to relax the task-specific matrix by preserving its last K elements. Here, the first/last K elements correspond to the gradients of the first/last K episodic memory tasks. 

As illustrated in Figure~\ref{fig:exp_AB3}, the PCA relaxation method outperforms FirstK and LastK under different value of K. This phenomenon is partly due to that PCA relaxation preserves more knowledge than other methods. This explanation can be confirmed by larger BWT of PCA Relaxation at different K. In addition, We observe that FirstK outperforms LastK in Figure~\ref{fig:exp_AB3}. This phenomenon may result from that tasks learnt earlier are easier to be forgotten than those tasks learnt later.

\section{Proof of $P$ is positive semidefinite}
$P=\mathrm{I}_{N} - BB^\top$, where $B\in\mathbb{R}^{N\times r}$, $r<N$ and $B^\top B=\mathrm{I}_r$. The following proof indicates that $P$ is positive semidefinite.
\begin{proof}
    Following singular value decomposition~\cite{noble1988applied}, $B=U\Sigma V^\top$, where $U\in\mathbb{R}^{N\times N}$ and $V\in\mathbb{R}^{r\times r}$ are unitary matrices, $\Sigma\in\mathbb{R}^{N\times r}$ is rectangular diagonal matrix with $r$ non-negative real numbers on the diagonal.
    
    We first show that $\Sigma^\top\Sigma=\mathrm{I}_r$. Since $B^\top B=\mathrm{I}_{r}$, $V\Sigma^\top\Sigma V^\top=\mathrm{I}_{r}$. Since $V$ is unitary, $\Sigma^\top\Sigma=\mathrm{I}_{r}$.
    
    We then show that $\forall v\in\mathbb{R}^N$, $v^\top BB^\top v\leq v^\top v$. Since $\Sigma$ is rectangular diagonal and $\Sigma^\top\Sigma=\mathrm{I}_{r}$, $\Sigma\Sigma^\top=\begin{bmatrix}\mathrm{I}_{r}&\mathbf{0}\\\mathbf{0}&\mathbf{0}\end{bmatrix}$. Therefore, 
    
    \begin{equation}
    \begin{aligned}
        v^\top BB^\top v&=v^\top U\Sigma\Sigma^\top U^\top v\\
        &=\Tr(v^\top U\Sigma\Sigma^\top U^\top v)\\
        &=\Tr(\Sigma\Sigma^\top U^\top vv^\top U)\\
        &\leq\Tr(U^\top vv^\top U)\\
        &=\Tr(v^\top U U^\top v)\\
        &=v^\top U U^\top v\\
        &=v^\top v.
       \end{aligned}
    \end{equation}
Thus, $\forall v\in\mathbb{R}^N$, $v^\top P v= v^\top v-v^\top BB^\top v\geq0$, which completes the proof.
\end{proof}

\section{Computational complexity reduction by LGU.} The main cost of our algorithm stems from Equ.(11) in the main text. We take (d) and (f) in Table~1 to analyze the time complexity of layerwise update. Denote $T$ as the total task number and $N$ as the number of model parameters. The total time cost of (d) is $O((T+1)N^2+(2T^2+1)N)$: $O(2T^2N)$ for calculating $B$ by Schmidt process, $O(TN^2)$ for calculating $P$, and $O(N^2+N)$ for getting $w$. Similarly, the total time cost of (f) is roughly $O((T+1)N^2/L+(2T^2+1)N)$, where $L$ is the number of network layers. Since $L > 1$, the time cost of (f) is possibly less than that of (d). By implementing 3-epochs setting experiments on MNIST, we found the running time of (d) and (f) for 20 tasks was 286.3s and 253.4s respectively, which empirically verified our analysis.

%%%%%%%%% ABSTRACT

{\small
\bibliographystyle{ieee}
\bibliography{egbib}
}